\documentclass{article}
\usepackage[preprint, nonatbib]{neurips}
\usepackage{graphicx} 
\usepackage{amsmath}
\usepackage{amssymb}
\usepackage{bbm}
\usepackage{hyperref}
\usepackage{authblk}
\usepackage{subcaption}
\usepackage{amsthm}

\theoremstyle{plain}
\newtheorem{theorem}{Theorem}[section]
\newtheorem{lemma}[theorem]{Lemma}

\title{Fixed Integral Neural Networks}

\author{\textbf{Ryan Kortvelesy} \\
University of Cambridge \\
Cambridge, UK \\
rk627@cam.ac.uk
}

\begin{document}

\maketitle

\section{Introduction}
\label{sec:Introduction}
It is often useful to perform integration over learned functions represented by neural networks \cite{monotonic, bayesian, neuralode}. However, this integration is usually performed numerically, as analytical integration over learned functions (especially neural networks) is generally viewed as intractable. In this work, we present a method for representing the \textit{analytical} integral of a learned function $f$. This allows the exact integral of a neural network to be computed, and enables constrained neural networks to be parametrised by applying constraints directly to the integral. Crucially, we also introduce a method to constrain $f$ to be positive, a necessary condition for many applications (\textit{e.g.} probability distributions, distance metrics, etc). Finally, we introduce several applications where our fixed-integral neural network (FINN) can be utilised.


\section{Method}
\label{sec:Method}
While it is not tractable to directly compute the integral of a function represented by a neural network, it is straightforward to take the analytical derivative. In this paper, we leverage the fundamental theorem of calculus in order to implicitly learn the integral of a function.

Suppose we wish to learn some function $f: \mathbb{R}^n \mapsto \mathbb{R}$. Instead of directly parametrising $f$, we represent it implicitly by parametrising its indefinite integral $F_\theta$ with a neural network:

\begin{equation}
    F_\theta(\vec{x}) = \int \int \cdots \int f(\vec{x}) \; dx_1 dx_2 \ldots dx_n
\end{equation}

Note that as $f$ is defined implicitly as a function of $F_\theta$, this is not an approximation---it is the \textit{exact} analytical integral. In order to solve for $f$, we must differentiate $F_\theta$:

\begin{equation}
    f(\vec{x}) = \frac{\partial}{\partial x_1} \frac{\partial}{\partial x_2} \cdots \frac{\partial}{\partial x_n} F_\theta(\vec{x})
    \label{eq:f_F}
\end{equation}

Although we parametrise its integral $F_\theta$, the function we wish to learn is $f$. Consequently, during the learning process, the loss is applied directly to $f$:

\begin{equation}
    \mathcal{L} = \mathbb{E}\left[ (y - f(\vec{x}))^2 \right]
\end{equation}

\begin{figure}
    \centering
    \begin{subfigure}[b]{.32\textwidth}
        \includegraphics[width=0.99\textwidth]{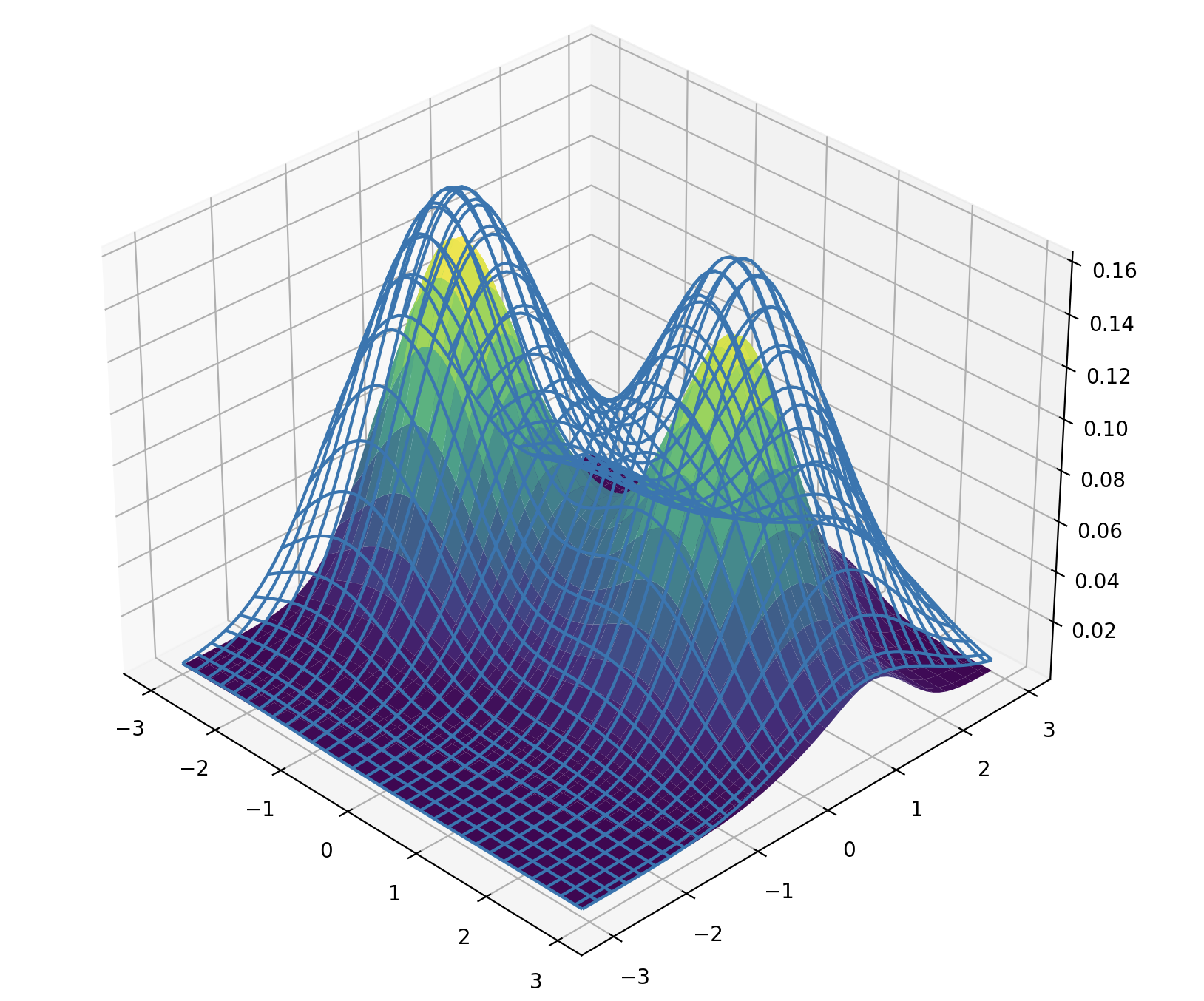}
        \caption{$\epsilon=1$}
    \end{subfigure}
    \begin{subfigure}[b]{.32\textwidth}
        \includegraphics[width=0.99\textwidth]{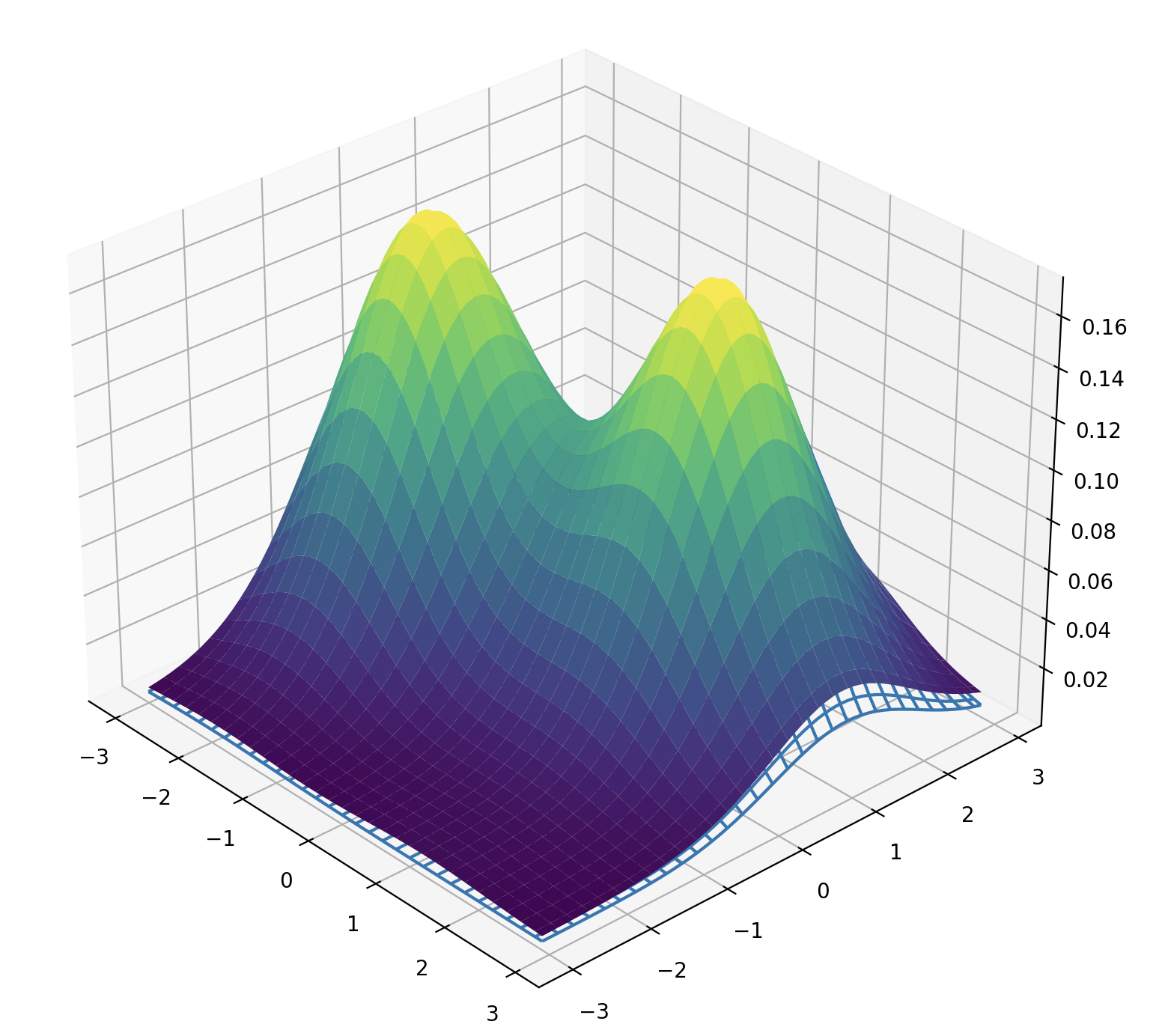}
        \caption{$\epsilon=2$}
    \end{subfigure}
    \begin{subfigure}[b]{.32\textwidth}
        \includegraphics[width=0.99\textwidth]{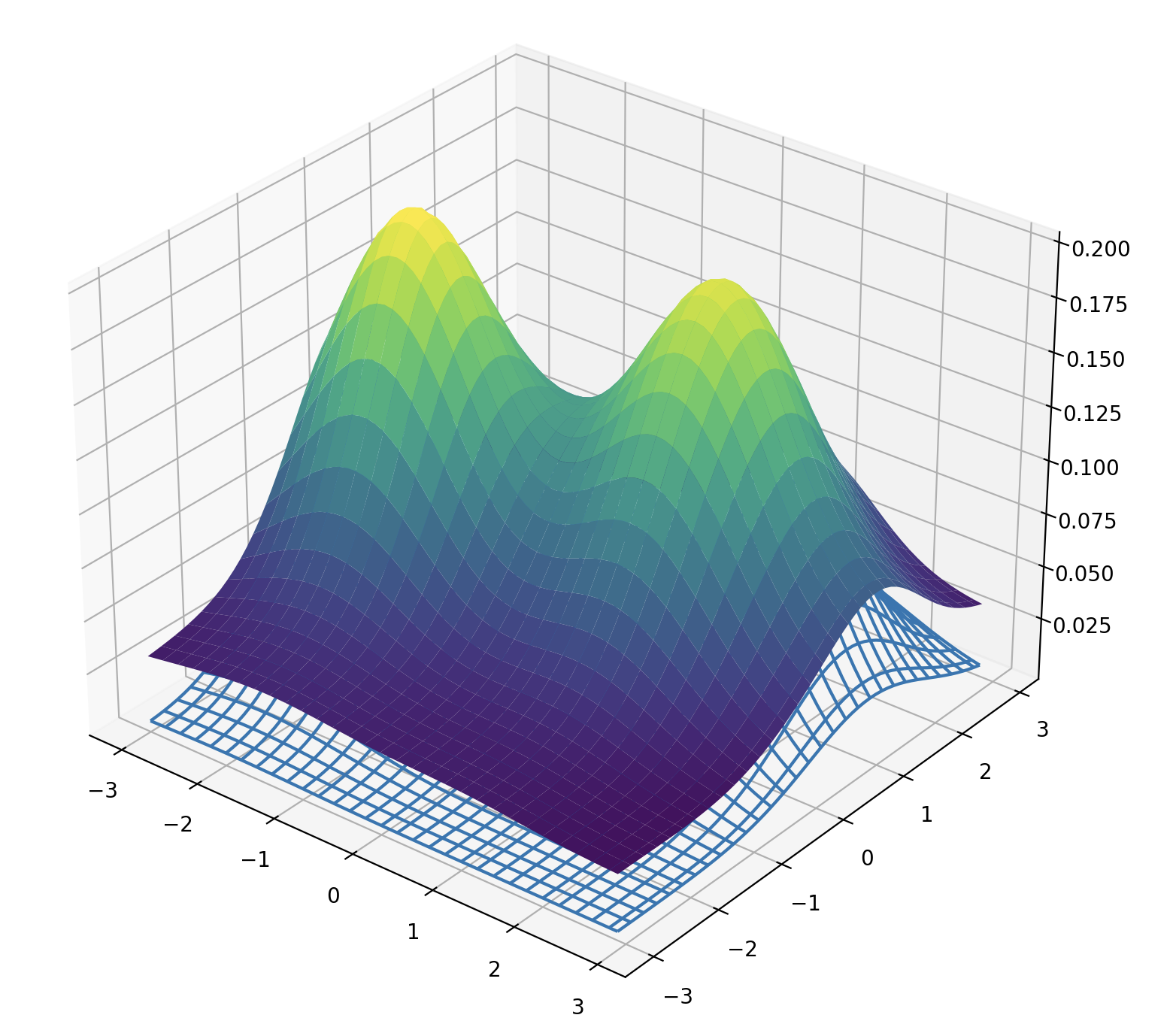}
        \caption{$\epsilon=3$}
    \end{subfigure}
    \caption{In this experiment, we create a Gaussian Mixture Model $g(x_1,x_2)$ with two modes and a total volume of $\int_{-3}^{3} \int_{-3}^{3} g(x_1,x_2) dx_1 dx_2 = 2$ as a ground truth (shown as a wireframe). We train three instances of FINN $f(x_1,x_2)$ to approximate the ground truth, with an integral constraint of $\epsilon=1$, $\epsilon=2$, and $\epsilon=3$ (less than, equal to, and greater than the volume of the ground truth). Note that although it is impossible for the constrained function under $\epsilon=1$ and $\epsilon=3$ to represent the ground truth, it converges to an approximation which minimises the error. Furthermore, note that this instance of FINN also enforces the positivity constraint, which restricts the function to positive values while still allowing negative derivatives.}
    \label{fig:training}
\end{figure}

\subsection{Integral Constraints}

Since $f$ is defined as a function of $F_\theta$, we can apply constraints directly to its integral. For example, by using an equality constraint, we can define the class of functions $f$ that integrate to a given value $\epsilon$ over a given domain $\mathcal{D}$. The same principle can be utilised to apply inequality constraints or transformations.

We start by considering rectangular domains defined by intervals $[a_i,b_i]$ for $i \in [1..n]$. The definite integral of $f$ over rectangular domain $\mathcal{D}$ is defined:

\begin{equation}
    F_\theta \Big\vert_\mathcal{D} = \sum_{p_1 \in [a_1, b_1]} \sum_{p_2 \in [a_2,b_2]} \cdots \sum_{p_n \in [a_n,b_n]} (-1)^{^{\sum \mathbbm{1}(p_i=a_i)}} \cdot F_\theta(\langle p_1, p_2, \ldots, p_n \rangle)
    \label{eq:int_eval}
\end{equation}

That is, in order to calculate the definite integral over an $n$-dimensional box, we must evaluate all of its vertices. This is because evaluating $F_\theta$ at a point will yield the ``area'' from negative infinity up to that point, so multiple points must be evaluated to determine the area of a finite region. The $(-1)$ exponent in the equation determines which regions must be subtracted and which must be added in order to determine that area (Figure \ref{fig:int_eval}).

\begin{figure}[h]
    \centering
    \includegraphics[width=0.4\textwidth]{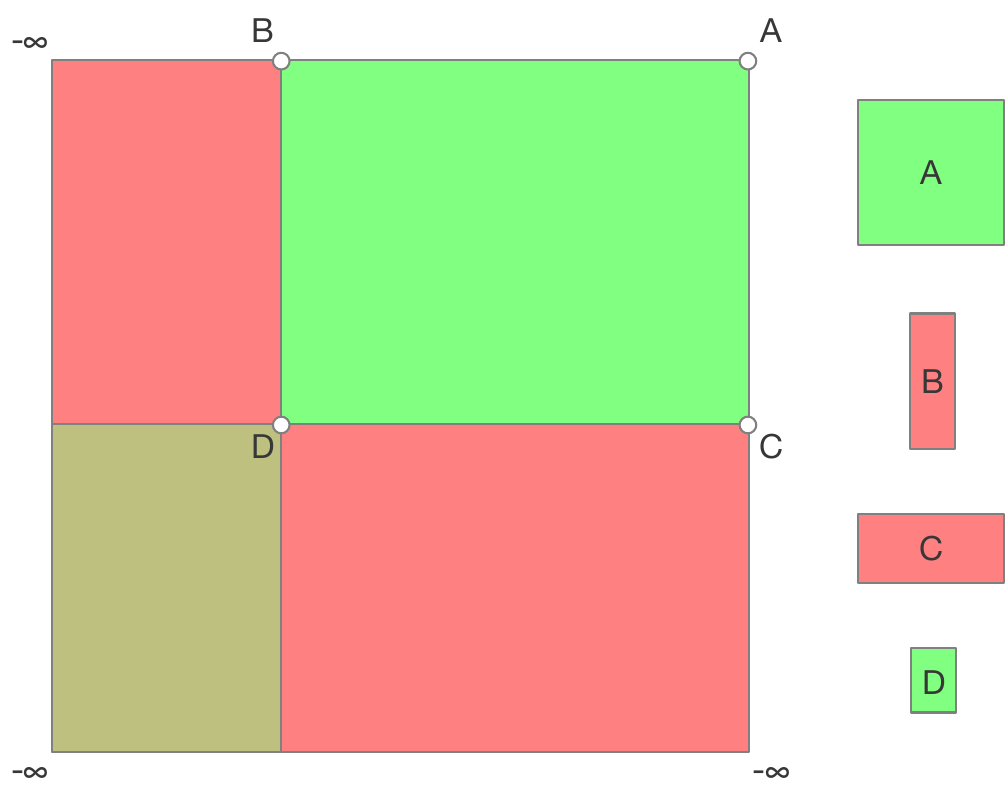}
    \caption{In order to determine the area bounded by points A and D, we must evaluate: $F(A)-F(B)-F(C)+F(D)$. This is because $F(A)$ returns the area from negative infinity up to point A, so $F(B)$ and $F(C)$ must be subtracted to confine the relevant region. However, as this subtracts the region up to point $D$ twice, $F(D)$ must be added to rectify the result. Equation $\eqref{eq:int_eval}$ generalises this formula to any dimension.}
    \label{fig:int_eval}
\end{figure}

In order to parametrise the class of functions which integrate to $\epsilon$, we start by defining $F'_\theta$, the integral of some unconstrained function $f'$. Then, we define the integral $F_\theta$ of our constrained function $f$ by rescaling $F'_\theta$:

\begin{equation}
    F_\theta(\vec{x}) = \frac{\epsilon}{F'_\theta \big\vert_\mathcal{D}} F'_\theta(\vec{x})
\end{equation}

Since the term $\frac{\epsilon}{F'_\theta \big\vert_\mathcal{D}}$ is a scalar, we can move it inside the integral:

\begin{align}
    F_\theta(\vec{x}) &= \frac{\epsilon}{F'_\theta \big\vert_\mathcal{D}} F'_\theta(\vec{x}) \\
    F_\theta(\vec{x}) &= \frac{\epsilon}{F'_\theta \big\vert_\mathcal{D}} \int \int \cdots \int f'(\vec{x}) \; dx_1 dx_2 \ldots dx_n \\
    F_\theta(\vec{x}) &= \int \int \cdots \int \frac{\epsilon}{F'_\theta \big\vert_\mathcal{D}} f'(\vec{x}) \; dx_1 dx_2 \ldots dx_n \\
\end{align}

Therefore, we can write the constrained $f$ as a function of the unconstrained integral $F'_\theta$, which carries the learnable parameters:

\begin{align}
    f(\vec{x}) &= \frac{\epsilon}{F'_\theta \big\vert_\mathcal{D}} f'(\vec{x}) \\
    f(\vec{x}) &= \frac{\epsilon}{F'_\theta \big\vert_\mathcal{D}} \cdot \frac{\partial}{\partial x_1} \frac{\partial}{\partial x_2} \cdots \frac{\partial}{\partial x_n} F'_\theta(\vec{x})
\end{align}

\subsection{Positivity Constraint}

In many applications of FINN, it is necessary to constrain $f$ to be non-negative. This is useful in cases where $f$ is only defined in the positive domain (see \autoref{sec:Applications}).

Following from \autoref{eq:f_F}, in order to apply this constraint, we must ensure that the mixed partial of $F_\theta$ is non-negative. To do this, we define a new neural network layer to construct our multi-layer perceptron (MLP):

\begin{equation}
    \sigma_n \left( \lvert W \! \rvert \, \vec{x} + b \right)
\end{equation}

In this layer, we apply an absolute value to the weights (but not the bias), and we use a custom activation function $\sigma_n$, which is conditioned on the dimension of the input. Note that although we wish the mixed partial of $F_\theta$ to be non-negative, it is too constraining to restrict further derivatives to be non-negative as well. The derivative of our function $\dot{f}$ (with respect to any input dimension) should be able to represent positive \textit{or} negative values. To satisfy these criteria, we define the following activation function using the error function $\mathrm{erf}(x) = \frac{2}{\sqrt{\pi}} \int_0^x e^{-t^2} dt$:

\begin{equation}
    \sigma_n = \underbrace{\int \int \cdots \int}_{n-1} \frac{\mathrm{erf}(x)+1}{2} \; \underbrace{dx \cdots \, dx \, dx}_{n-1}
\end{equation}

For $n=1$, this simplifies to $\frac{\mathrm{erf}(x)+1}{2}$, which closely resembles sigmoid. For $n=2$, it resembles softplus, which is the integral of sigmoid. While they are similar, it is crucial that we use this custom activation instead of sigmoid, because the higher-order integrals of sigmoid evaluate to the polylogarithm, which has no closed-form solution. Conversely, all of the integrals of the error function have analytical solutions in terms of linear compositions of constants, power functions $x^k$, exponentials $e^{-x^2}$, and the error function itself $\mathrm{erf}(x)$ (all of which have efficient implementations). In practice, we use symbolic math to compute the integral once at initialisation time, and then each forward pass evaluates the resulting expression.

\section{Applications}
\label{sec:Applications}
There are many possible applications for FINN. In this section, we highlight four use cases, including 1) distance metrics between functions, 2) trajectory optimisation, 3) probability distributions, and 4) optimal bellman updates in reinforcement learning.

\subsection{Distance Metrics}

\begin{figure}
    \centering
    \includegraphics[width=0.4\textwidth]{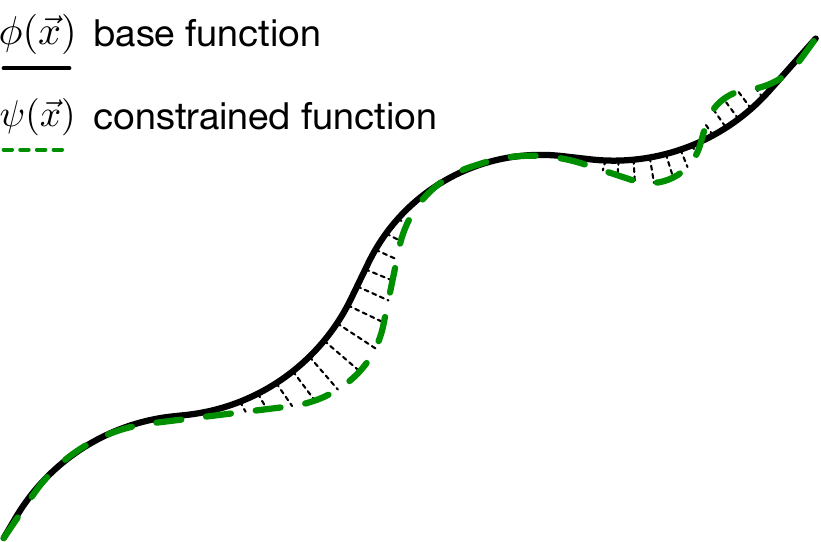}
    \caption{In the \textit{Distance Metrics} application, we can use FINN to not only calculate, but constrain an integral-based metric between two functions. In this diagram, we show a base function $\phi$ and a learned function $\psi$ which is constrained such that $\int_\mathcal{D} d(\psi(\vec{x}), \phi(\vec{x})) \; d\vec{x} = \epsilon$. Note that although we visualise in two dimensions, this method is valid for any number of dimensions.}
    \label{fig:metric}
\end{figure}

Given a base function $\phi(x): \mathbb{R}^n \mapsto \mathbb{R}^m$, consider the task of defining a new function $\psi(x): \mathbb{R}^n \mapsto \mathbb{R}^m$ with bounded deviation from $\phi$ (\autoref{fig:metric}). While there are several methods to quantify similarity between functions, we select the ``integral metric'', which takes the integral of a distance metric $d : \mathbb{R}^{2n} \to \mathbb{R}$ between the two functions over some domain:

\begin{equation}
    J(\psi, \phi) = \int d(\phi(\vec{x}), \psi(\vec{x})) \; d\vec{x}
\end{equation}

In this case, we select $d(a,b) = ||a-b||$. If we constrain $\psi$ to represent the class of functions such that $J(\phi,\psi) \leq \epsilon$, then it can only place a finite amount of mass across the domain. The constrained function $\psi$ can be implemented by adding a direction function $\nu : \mathbb{R}^n \to \mathbb{R}^m$ multiplied by a state-dependent scaling factor $f : \mathbb{R}^n \to \mathbb{R}$ defined by FINN to the base function $\phi$:

\begin{equation}
    \psi(\vec{x}) = \phi(\vec{x}) + \frac{f(\vec{x})}{\lvert\lvert \nu(\vec{x}) \rvert\rvert} \cdot \nu(\vec{x})
\end{equation}

Then, if we select a our fixed integral constant to be $\int f(\vec{x}) d\vec{x} \leq \epsilon$ (using an inequality constraint, and enforcing positivity of $f$), we can show that $J(\phi, \psi) \leq \epsilon$:

\begin{align}
    J(\phi, \psi) &= \int \Big\lvert\Big\lvert \phi(\vec{x}) - (\phi(\vec{x}) + \frac{f(\vec{x})}{\lvert\lvert \nu(\vec{x}) \rvert\rvert} \cdot \nu(\vec{x})) \Big\rvert\Big\rvert \; d\vec{x} \\
     &= \int \Big\lvert\Big\lvert -f(\vec{x}) \cdot \frac{\nu(\vec{x})}{\lvert\lvert \nu(\vec{x})) \rvert\rvert} \Big\rvert\Big\rvert \; d\vec{x} \\
     &= \int f(\vec{x}) \cdot \frac{\lvert\lvert \nu(\vec{x})) \rvert\rvert}{\lvert\lvert \nu(\vec{x})) \rvert\rvert} \; d\vec{x} \\
     &= \int f(\vec{x}) \; d\vec{x} \\
     &\leq \epsilon \\
\end{align}

There are many possible applications for this---for example, it could be used to add a bounded (and controllable) amount of representational complexity on top of an explainable or ground truth model. Such an architecture could mitigate issues with overfitting, forcing a model to learn a representation which distributes its mass in the most efficient manner.

\subsection{Trajectory Optimisation}

\begin{figure}[h]
    \centering
    \includegraphics[width=0.5\textwidth]{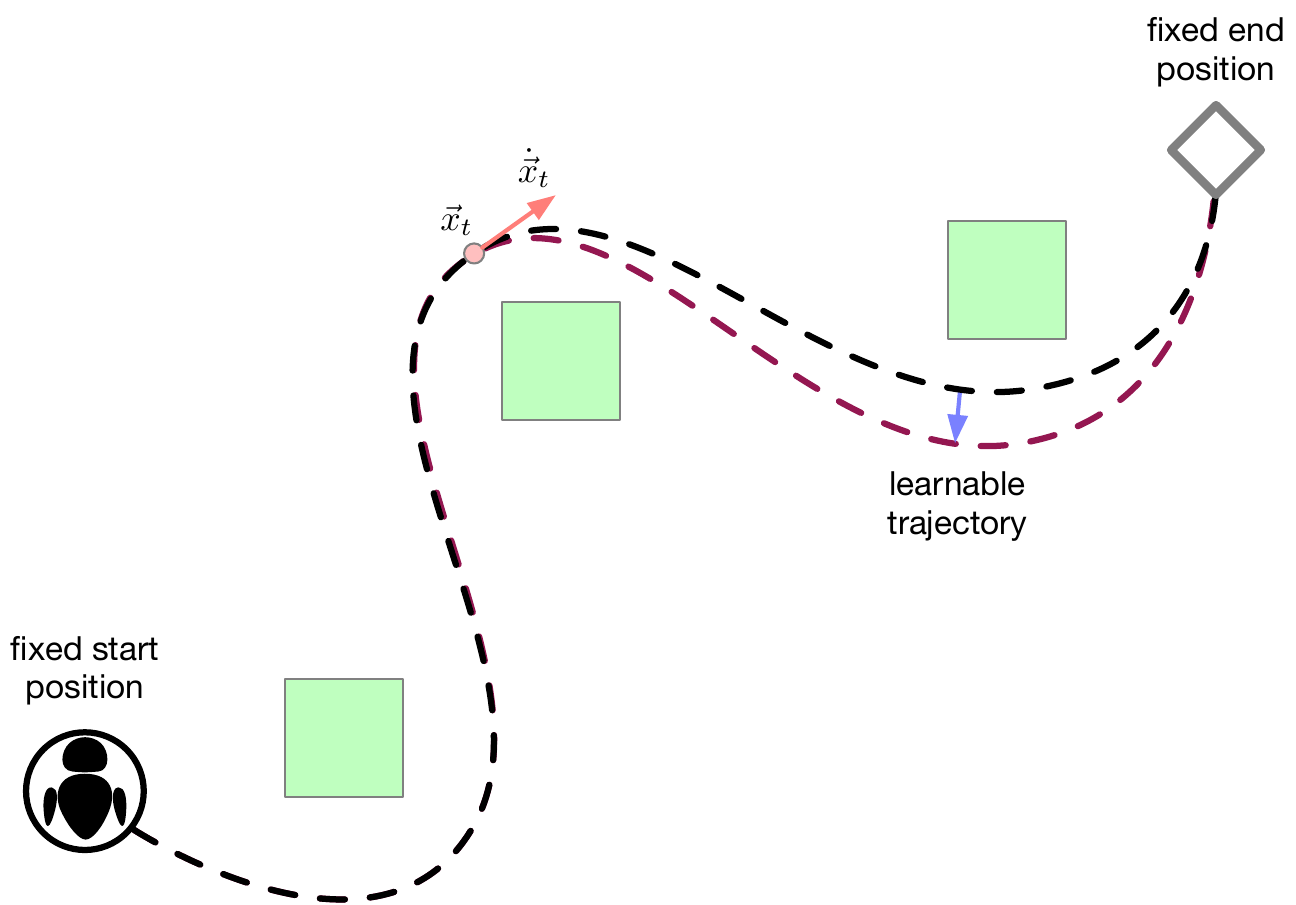}
    \caption{In the \textit{Trajectory Optimisation} application, we demonstrate how FINN can be used to parameterise a trajectory which is constrained to start and end in given positions. This trajectory can then be optimised to satisfy dynamics or collision constraints (treated as objectives). Either state $\vec{x}$ or its time derivative $\dot{\vec{x}}$ can be queried from the resulting open-loop trajectory.}
    \label{fig:traj}
\end{figure}

Given some physical system in state $x_0$ suppose we wish to define a trajectory $\dot{x}(t)$ to bring the system to state $x_T$. Furthermore, suppose we have access to an inverse dynamics model $u = T(x, \dot{x})$.

If we define $\vec{F}_\theta(t) = \langle F_{\theta_1}(t), F_{\theta_2}(t), F_{\theta_3}(t) \rangle: \mathbb{R} \mapsto \mathbb{R}^3$ with three instances of FINN, then we can parametrise the class of trajectories which end in the desired end state $x_T$. To do this, we set the integral equality constraint $\epsilon$ for each $F_{\theta_i}$ to the corresponding dimension of the terminal state $\langle \epsilon_1, \epsilon_2, \epsilon_3 \rangle = x_T - x_0$. Note that in this application we \textit{do not} apply the positivity constraint, as it is permissible for $\dot{x}(t)$ to assume negative values.

With this parametrisation, we get a trajectory $x(t) = \vec{F}_\theta(t) - \vec{F}_\theta(t_0) + x_0$ and its corresponding first order derivative $\dot{x}(t) = f(t)$ which are constrained to start at $x_0$ and end at $x_T$. Using the inverse dynamics model, an open-loop control trajectory can be defined $u(t) = T(x(t), \dot{x}(t))$.

During the optimisation process the trajectory can be updated, but it will always be constrained to produce $\dot{x}(t)$ which passes through $x_0$ and $x_T$. Since the representation of $x(t)$ and $\dot{x}(t)$ share the same parameters, it is possible to apply loss to either $x(t)$ \textit{or} $\dot{x}(t)$. For example, the derivative $\dot{x}(t)$ can be optimised to satisfy dynamics constraints or minimise energy usage. Alternatively, the trajectory itself $x(t)$ can be optimised to avoid obstacles.

\subsection{Probability Distributions}

\begin{figure}[h]
    \centering
    \includegraphics[width=0.5\textwidth]{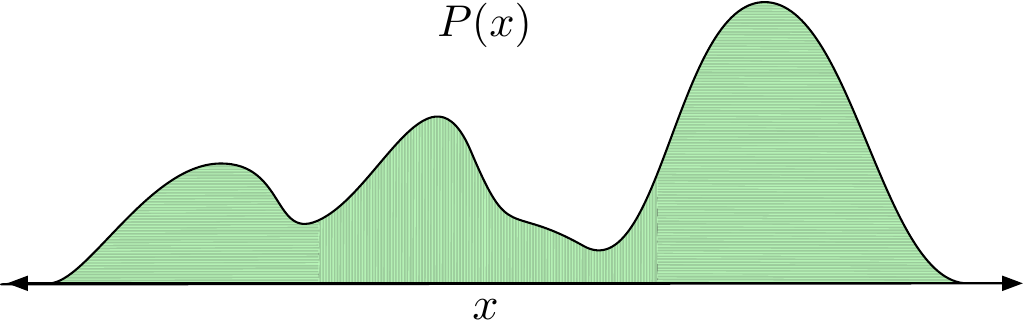}
    \caption{In the \textit{Probability Distributions} application, we use FINN to model arbitrary continuous probability distributions. Note that while we visualise our distribution in one dimension, our method works for any number of dimensions.}
    \label{fig:prob}
\end{figure}

As FINN can parametrise the class of positive functions which integrate to $1$, it can be used to represent arbitrary probability distributions. In this context, FINN can be viewed as an alternative to Normalising Flows \cite{normflows}.

The implementation of a probability density function $p: \mathbb{R}^n \mapsto \mathbb{R}$ is fairly straighforward---if we impose the positivity constraint and an integral equality constraint with $\epsilon=1$ over some domain $\mathcal{D}$, then $p(\vec{x}) = f(\vec{x})$. By extension, the integral $F_\theta$ represents the cumulative density function of $p$.

While evaluating the PDF and CDF of the learned probability distribution is straightforward, sampling from the distribution presents a challenge. The standard method for sampling from a multivariate distribution involves sampling from a uniform distribution $t \sim \mathcal{U}[0,1]$, using the inverse CDF to produce a sample in a single dimension, and then repeating the process in each other dimension using conditional probability distributions. Unfortunately, we do not have a representation of the inverse CDF. However, the learned CDF represented by $F_\theta$ is monotonic, so it is straighforward to compute the inverse with binary search or Newton's method.

\subsection{Bellman Optimality Operator}

\begin{figure}[h]
    \centering
    \includegraphics[width=0.6\textwidth]{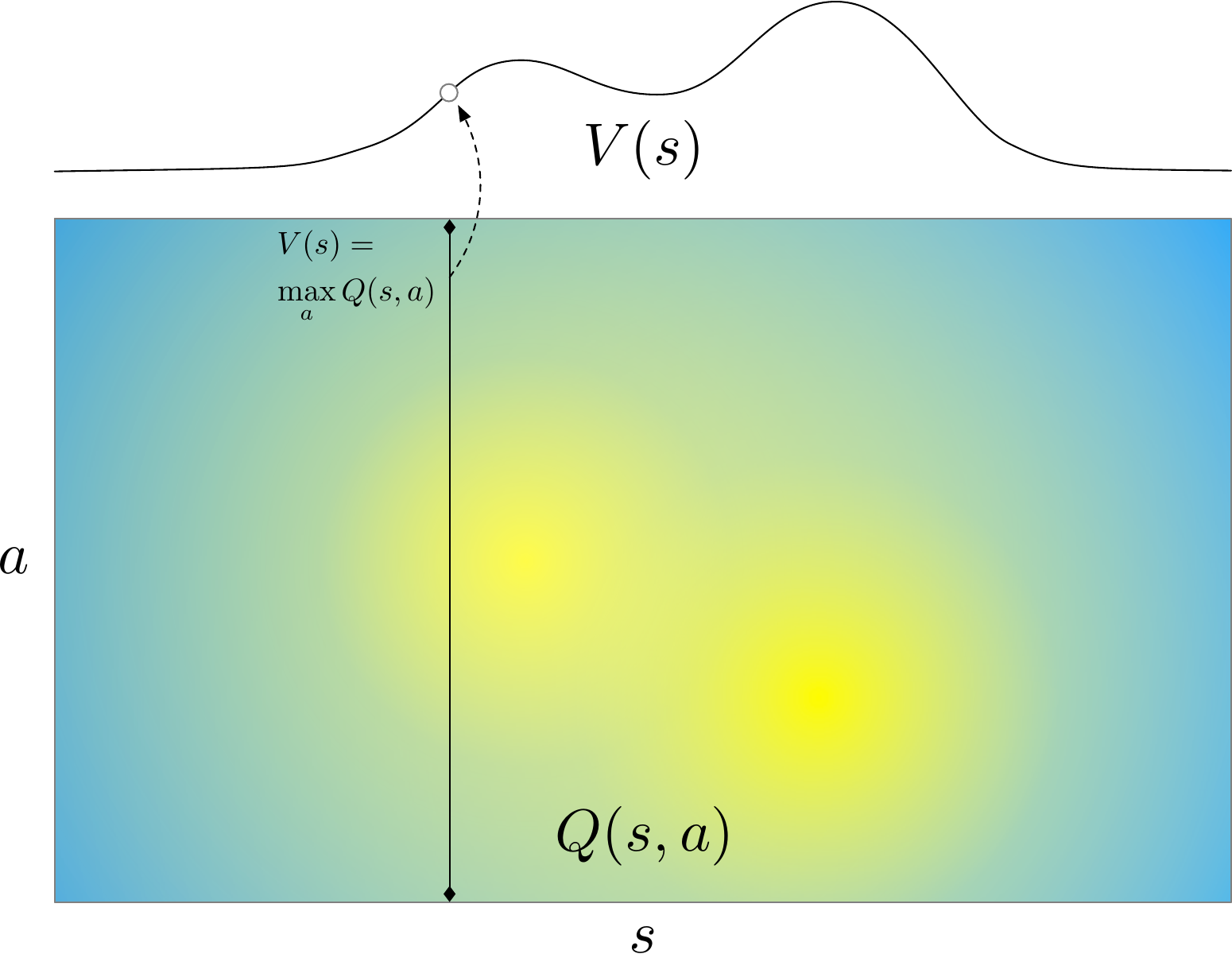}
    \caption{In the \textit{Bellman Optimality Operator} application, we use FINN in the context of reinforcement learning to calculate optimal value and action-value functions. Usually, in the continuous domain, it is not possible to take the maximum over all actions. However, we can use FINN to integrate over actions in the continuous domain using the \textit{soft bellman operator}.}
    \label{fig:rl}
\end{figure}

In reinforcement learning, the Bellman optimality equation defines the value of a given state as the expected discounted sum of future rewards that would be accrued by an optimal policy, starting at that state. The Bellman operator is a dynamic programming update rule which is applied to a value function such that it approaches the optimal value function:

\begin{align}
    \mathcal{B}[Q(s,a)] &= \mathbb{E} \left[ 
r(s,a,s') + \gamma \max_a Q(s',a) \right]
\end{align}

Note that in this equation we must compute the maximum over all actions to compute the value for the next state $V(s') = \max_a Q(s',a)$. While this is possible in discrete settings, it is generally intractable in continuous spaces. Consequently, most approaches sacrifice some optimality for computational feasibility, opting for an on-policy update (where $\pi$ is the current policy):

\begin{align}
    \mathcal{B}[Q(s,a)] &= \mathbb{E} \left[ 
r(s,a,s') + \gamma \, Q(s',\pi(s)) \right]
\end{align}

However, while the first update rule converges to the optimal value function regardless of which actions are chosen, the on-policy approach is biased by the current policy. If the policy takes suboptimal actions, then the value function will be biased towards the values associated with those actions.

However, FINN enables the use of the optimal Bellman operator in continuous space, as it allows us to analytically integrate over actions. Consequently, it converges towards the optimal Q-values regardless of which policy is used.

As a preliminary, note that the $\max$ can be represented with $\mathrm{log\text{-}sum\text{-}exp}$ using temperature parameter $\beta$. This gives rise to the \textit{soft bellman operator} \cite{softbellman}:

\begin{align}
    \mathcal{B}[Q(s,a)] &= \mathbb{E} \left[ 
r(s,a,s') + \gamma \frac{1}{\beta}\log\left( \int 
e^{\beta Q(s,a)} da \right) \right]
\end{align}

If we select $f(s,a) = e^{\beta Q(s',a)}$ and integrate over the actions, we get $F_\theta(s,a) = \int 
e^{\beta Q(s',a)} da$. So, we can represent the value functions as:

\begin{align}
    Q(s,a) &= \frac{1}{\beta} \log(f(s,a)) \\
    V(s) &= \frac{1}{\beta} \log(F_\theta(s,a))
\end{align}

Note that we do not apply an integral constraint, but we do apply a positivity constraint to ensure that $e^{\beta Q(s',a)}$ is positive (and therefore the $\log$ is a real number). However, since we only wish $f$ to be monotonic with respect to the actions $a$ (not the state $s$), we cannot use the standard MLP that we defined in \autoref{sec:Method}. Instead, we use a separate hypernetwork conditioned on $s$ to generate the parameters of our monotonic MLP, allowing $f$ to be non-monotonic with respect to $s$.

\section{Conclusion}
\label{sec:Conclusion}
In this paper we introduced FINN, a method for computing the analytical integral of a learned neural network. Our approach defines several constraints and extensions crucial to potential applications, including 1) an integral constraint, which restricts $f$ to the set of functions which integrate to a given value, 2) a positivity constraint, which restricts $f$ to be positive, and 3) an extension to aribtrary domains, which allows integration over regions which cannot be represented by an $n$-dimensional box. To demonstrate the use cases of FINN, we also present four applications: distance metrics, trajectory optimisation, probability distributions, and bellman optimality operators.

\newpage
\bibliographystyle{unsrt}
\bibliography{ref}

\begin{thebibliography}{1}

\bibitem{monotonic}
Antoine Wehenkel and Gilles Louppe.
\newblock Unconstrained {Monotonic} {Neural} {Networks}.
\newblock In {\em Advances in {Neural} {Information} {Processing} {Systems}},
  volume~32. Curran Associates, Inc., 2019.

\bibitem{bayesian}
Charles Blundell, Julien Cornebise, Koray Kavukcuoglu, and Daan Wierstra.
\newblock Weight {Uncertainty} in {Neural} {Networks}, May 2015.
\newblock arXiv:1505.05424 [cs, stat].

\bibitem{neuralode}
Ricky T.~Q. Chen, Yulia Rubanova, Jesse Bettencourt, and David Duvenaud.
\newblock Neural {Ordinary} {Differential} {Equations}, December 2019.
\newblock arXiv:1806.07366 [cs, stat].

\bibitem{normflows}
Ivan Kobyzev, Simon J.~D. Prince, and Marcus~A. Brubaker.
\newblock Normalizing {Flows}: {An} {Introduction} and {Review} of {Current}
  {Methods}.
\newblock {\em IEEE Transactions on Pattern Analysis and Machine Intelligence},
  43(11):3964--3979, November 2021.
\newblock arXiv:1908.09257 [cs, stat].

\bibitem{softbellman}
Tuomas Haarnoja, Haoran Tang, Pieter Abbeel, and Sergey Levine.
\newblock Reinforcement {Learning} with {Deep} {Energy}-{Based} {Policies},
  July 2017.
\newblock arXiv:1702.08165 [cs].

\end{thebibliography}

\newpage
\appendix
\section{Appendix}
\subsection{Positivity Proof}

To constrain $f$ to be non-negative, we construct an MLP with a composition of $L$ custom layers $\psi_i$:

\begin{align}
    \label{eq:F_custom}
    F_\theta(\vec{x}) &= \psi_L \left(\ldots \psi_2\left(\psi_1(\vec{x})\right)\right) \\
    \psi_i &= \sigma_n \left( \lvert W_i \! \rvert \, \vec{x} + b_i \right)
    \label{eq:psi_custom}
\end{align}

The activation function $\sigma_n$ denotes a custom nonlinearity defined as a function of $n$, the dimension of the input $\vec{x} \in \mathbb{R}^n$:

\begin{equation}
    \sigma_n = \underbrace{\int \int \cdots \int}_{n-1} \frac{\mathrm{erf}(x)+1}{2} \; \underbrace{dx \cdots \, dx \, dx}_{n-1}
    \label{eq:sigma_n}
\end{equation}

Since $f$ is defined implicitly through our parametrisation of $F_\theta$, we must show that the mixed partial of $F_\theta$ is positive to ensure that $f$ is non-negative. From \autoref{eq:f_F}, we can write:

\begin{align}
    \frac{\partial}{\partial x_1} \frac{\partial}{\partial x_2} \cdots \frac{\partial}{\partial x_n} F_\theta(\vec{x}) \geq 0 &\implies f(\vec{x}) \geq 0
\end{align}

\vspace{5mm}

\begin{lemma}
    \label{lemma:multilayer}
    If the derivative of each layer $\psi_j$ with respect to some $x_i$ is non-negative, then the derivative of $F_\theta$ is non-negative: $\forall j, \frac{\partial}{\partial x_i} \psi_j(\vec{x}) \geq 0 \implies \frac{\partial}{\partial x_i} F_\theta(\vec{x}) \geq 0$.
\end{lemma}
\begin{proof}
    Using the chain rule, we see that the partial derivative of $F_\theta$ is the product of the partials of each layer $\psi_j$, evaluated at the output of the preceding layer:

    \begin{align}
        \frac{\partial}{\partial x_i} \psi_L \left(\ldots \psi_2\left(\psi_1(\vec{x})\right)\right) = \; &\frac{\partial \psi_L}{\partial x_i} \left(\ldots \psi_2\left(\psi_1(\vec{x})\right)\right) \\
        &\;\;\;\;\;\;\;\;\;\;\vdots \nonumber \\
        & \cdot \frac{\partial \psi_2}{\partial x_i}(\psi_1(\vec{x})) \nonumber \\
        & \cdot \frac{\partial \psi_1}{\partial x_i}(\vec{x}) \nonumber
    \end{align}

    So, if the derivative of each layer $\frac{\partial \psi_j}{\partial x_i}$ is positive, then the derivative of the entire network $\frac{\partial}{\partial x_i} F_\theta$ is positive.
\end{proof}

\begin{theorem}
    Under the parametrisation of $F_\theta$ in \autoref{eq:F_custom}, the function $f$ is non-negative.
\end{theorem}

\begin{proof}
    Consider a single layer of $F_\theta$, denoted $\psi$, given by \autoref{eq:psi_custom}. Furthermore, without loss of generality, let us consider a single dimension of the output $j$, denoted $\psi_j$.

    If we take the partial derivative with respect to input $x_p$, we can apply the chain rule to obtain:

    \begin{align}
        \frac{\partial}{\partial x_p} \psi_j(\vec{x}) &= \frac{\partial}{\partial x_p} \left[ \sigma_n \left( \sum_{i \in [1..n]} w_{ij} \cdot x_i \right) \right] \nonumber \\
         &= \frac{\partial \sigma_n}{\partial x_p} \left( \sum_{i \in [1..n]} w_{ij} \cdot x_i \right) \cdot \frac{\partial}{\partial x_p} \left[ \sum_{i \in [1..n]} w_{ij} \cdot x_i \right] \nonumber \\
         &= \dot{\sigma}_n \left( \sum_{i \in [1..n]} w_{ij} \cdot x_i \right) \cdot w_{pj}
    \end{align}
    
    If we then take the partial derivative with respect to input $x_q$ to form the mixed partial $\frac{\partial^2}{\partial x_p \partial x_q}$, we get:

    \begin{equation}
        \frac{\partial}{\partial x_q} \left(\frac{\partial}{\partial x_p} \psi_j(\vec{x}) \right) = \ddot{\sigma}_n \left( \sum_{i \in [1..n]} w_{ij} \cdot x_i \right) \cdot w_{pj} \cdot w_{qj}
    \end{equation}
    
    Finally, if we take the derivative with respect to 
    \textit{all} inputs, the result is:

    \begin{equation}
        \frac{\partial^n}{\partial x_1 \partial x_2 \ldots \partial x_n} \psi_j(\vec{x}) = \sigma^{(n)}_n \left( \sum_{i \in [1..n]} w_{ij} \cdot x_i \right) \cdot \prod_{i \in [1..n]} w_{ij}
    \end{equation}
    
    We know that our weights are all positive $w_{ij} \geq 0$, and our activation function is parametrised such that it is positive up to the $n$-th derivative (see \autoref{eq:sigma_n}): $\forall k \in [1..n],\forall x \in \mathbb{R}: \sigma^{(k)}_n(x) > 0$. Consequently, the mixed derivative of $\psi_j$ must be positive: $\frac{\partial^n}{\partial x_1 \partial x_2 \ldots \partial x_n} \psi_j(X) \geq 0$.
    
    Furthermore, by applying Lemma \ref{lemma:multilayer}, we can show that since the mixed partial of each layer is positive, the mixed partial of the entire network is also positive:
    
    $$\frac{\partial^n}{\partial x_1 \partial x_2 \ldots \partial x_n} F_\theta(\vec{x}) \geq 0$$
\end{proof}

\end{document}